\documentclass{amsart}

\usepackage{amsmath,amsthm,amscd,amssymb,eucal,mathrsfs}
\usepackage{amsfonts}
\usepackage{dsfont}
\usepackage{latexsym}
\usepackage{graphicx}
\usepackage{multicol}

\usepackage[all]{xy}

\newtheorem{thm}{Theorem}[section]

\newtheorem{lem}[thm]{Lemma}

\newtheorem{rem}[thm]{Remark}
\newtheorem{dfn}[thm]{Definition}

\newtheorem{exa}[thm]{Example}

\newtheorem{alg}[thm]{Algorithm}

\newtheorem*{Satz*}{Satz}

\newtheorem*{Ass}{Assumption}

\newcommand{\mathset}[1]{{\left\{#1\right\}}} 
\newcommand{\absolute}[1]{\left\lvert#1\right\rvert}

\newcommand{\children}{{\rm  ch}}
\newcommand{\Vertices}{{\rm Vert}}
\newcommand{\Ends}{{\rm Ends}}

\begin{document}

\title{On $p$-adic classification}

\author{Patrick Erik Bradley}
 
\date{\today}

\begin{abstract}
A $p$-adic modification of the split-LBG classification method is presented
in which first clusterings and then cluster centers  are computed which locally
minimise
an energy function. The outcome for a fixed dataset 
is independent of the prime number $p$ with
finitely many exceptions.
The methods are applied to the construction of
$p$-adic classifiers in the context of learning.
\end{abstract}

\maketitle

\section{Introduction}

The field $\mathds{Q}_p$ of  $p$-adic numbers is of interest  in hierarchical classification 
because of its inherent hierarchical structure \cite{Perruchet1982}. 
A great amount of work deals with finding $p$-adic data representation
(e.g.\ \cite{Murtagh-JoC2004, Murtagh2008}).

In \cite{Brad-JoC}, the use of more general  $p$-adic numbers for encoding
hierarchical
data was advocated in order to be able to include the case of non-binary
dendrograms
into the scheme without having to resort to a larger prime number $p$.
This was applied in \cite{Brad-TCJ} to the special case of data consisting
in words over a given alphabet and where proximity of words is defined by
the length of the common initial part. There,  an agglomerative 
hierarchic $p$-adic
clustering algorithm was described. However, the question of finding optimal
clusterings of $p$-adic data was not raised.

Already in \cite{B-PXK2001}, the performance of
 classical and $p$-adic classification algorithms was compared in the
 segmentation of moving images. It was observed that the $p$-adic ones were
often more efficient. Learning algorithms using $p$-adic neural networks are
described in \cite{B-PX2007,XT2009}.

\medskip
Inspired by  \cite{B-PXK2001},
our main concern in this article will be a $p$-adic adaptation of the
so-called split-LBG method  which finds energy-optimal
clusterings of data. The name ``LBG'' refers to the initials of the
authors of \cite{LBG1980}, where it is described first. Their method is to
 find cluster
centers, and then to group the data around the centers. In the next step,
the cluster centers are split, and more clusters are obtained. This process
is repeated until the desired class number is attained.
For $p$-adic data, this approach does not make sense: 
first of all, cluster centers are in general not unique; and secondly,
because the dendrogram
is already determined by data, an arbitrary choice of cluster centers is not
possible---this can lead to incomplete clusterings. Hence, we first find
clusterings by refining in the direction of highest energy reduction,
until the class number exceeds a prescribed bound.
Thereafter, candidates for cluster centers are computed: they minimise
the cluster energy. The result is a sub-optimal method for 
$p$-adic classification which splits a given cluster into its maximal proper
subclusters.
A variant discards first all quasi-singletons, i.e.\ clusters of  energy
below a threshold value. 
The {\em a posteriori} choice of centers turns out useful for 
constructing 
classifiers.

\medskip
A first application of some of the methods described here to event history data of building
stocks is described in \cite{Brad-buildevp}. There, 
the classification algorithm is performed on different $p$-adic encodings
of the data in order to compare the dynamics of some sampled municipal building stocks.

\medskip
After introducing notations in Section \ref{sec-general}, we briefly describe
the
classical split-LBG method in Section \ref{sec-LBG}.
Section \ref{sec-LBGp} reformulates the minimisation task of split-LBG in 
the $p$-adic setting, and describes the corresponding algorithms.
The issue on the choice of the prime $p$ is dealt with in Section
\ref{anyprime}.
 Section \ref{sec-learn}  constructs classifiers and presents an adaptive
learning method in which accumulated clusters of large energy are split.

\section{Generalities} \label{sec-general}

\subsection{$p$-adic numbers}
Let $p$ be a prime number, and
$K$ a  field which is a finite extension field of the field $\mathds{Q}_p$ of
rational $p$-adic numbers. We call the elements of $K$ simply $p$-adic numbers. 
$K$ is a normed field whose norm $\absolute{\ }_K$ extends the $p$-adic norm
$\absolute{\ }_p$ on $\mathds{Q}_p$. 
Let
$\mathcal{O}_K:=\mathset{x\in K\mid \absolute{x}_K\le 1}$ denote the local ring of integers of
$K$. Its maximal ideal $\mathfrak{m}_K=\mathset{x\in K\mid \absolute{x}_K<1}$
 is generated by
a {\em uniformiser} $\pi$. It has the property $v(\pi)=\frac{1}{e}$, 
where $e\in\mathds{N}$ is the ramification degree of $K/\mathds{Q}_p$. 

\smallskip
All elements  $x\in K$ have a $\pi$-adic expansion
\begin{align}
x=\sum\limits_{i\ge-m} \alpha_i\pi^i \label{pi-adic-expansion}
\end{align}
with coefficients $\alpha_i$ in some set $\mathcal{R}\subseteq K$ of
representatives for the residue field
$O_K/\mathfrak{m}_K\cong\mathds{F}_{p^f}$.  
In the case $q=p$, the choice $\mathcal{R}=\mathset{0,1,\dots,p-1}$ 
is quite often made.

\smallskip
By $X$ will will always mean a finite set of data taken from $K$.

\subsection{$p$-adic clusters}

A {\em disk} in some finite set $X\subseteq K$ is a subset of the form
$$
\mathset{x\in X\mid \absolute{x-a}_K<\varepsilon}
$$
for some $a\in X$ and $\varepsilon>0$.
In particular, any singleton $\mathset{x}\subseteq X$ is a disk in $X$.

\medskip
The {\em cluster property} of a subset $C$ of $p$-adic data $X\subseteq K$ is given by saying that
for any $a\in C$ 
it holds true that
\begin{align}
\absolute{x-a}_K<\mu(C) \Rightarrow x\in C, \label{clusterproperty}
\end{align}
where 
$$
\mu(C):=\max\mathset{\absolute{x-y}_K\mid x,y\in C}
$$
is the cluster diameter.
As a consequence, a cluster is a union of disks in $X$.
We will call a disk in $X$ also a {\em verticial cluster},
because in the in the dendrogram for $X$, the vertices correspond to those
clusters which are (non-singleton)\footnote{In many definitions of 
dendrograms, the data correspond to terminal vertices, but in
our definition in Section \ref{somedefs}, data are not considered
as vertices of the dendrogram. Nevertheless, we do not exlude singleton clusters
from the definition of ``vertcial''. We apologise for this inconsistency.}  disks. 
More to the dendrogram associated to $p$-adic
data will be said in Section \ref{somedefs}. In Figure \ref{non-clust}
 the ultrametric property of dendrograms is visualised
 as follows: data  $b,c$ connected by a  path consisting of  vertical and  horizontal line segments
are considered as near, if the
sum of the  vertical parts is short. A third datum $a$ further away from
 $b$ and $c$ is, by ultrametricity, at equal distance to $b$ and $c$.
This fact is visualised by having  paths $a\leadsto b$ and $a\leadsto c$
with  vertical components summing up to equal length.

\begin{exa}
Let $X=\mathset{a,b,c}$, and consider the subset $C=\mathset{a,b}$.
In Figures \ref{non-clust} and \ref{nonvertclust}, we assume two different
dendrograms
for our data $X$. 
In Figure \ref{non-clust}, the disks are the singletons,
the set $\mathset{b,c}$, and the whole dataset $X$.
Hence, $C$ is not a cluster in the case of Figure
\ref{non-clust}, because it does not satisfy the cluster property (\ref{clusterproperty}): $b$ and $c$ are at distance less than the diameter which equals
the distance between $a$ and $b$, whereas $C$ contains
$b$ but not $c$.
However, in Figure \ref{nonvertclust}, 
all data are at equal distance, so the only disks are the singletons 
and $X$.
Hence,  $C$ is a cluster in Figure \ref{nonvertclust}, 
 but not a disk,
i.e.\ not verticial.

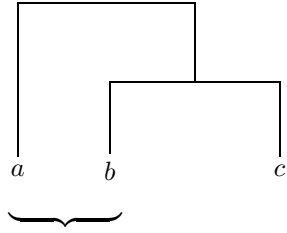
\begin{figure}
$$
\begin{array}{c}
\xymatrix{
*\txt{}\ar@{-}[rr]\ar@{-}[dd]&&*\txt{}\ar@{-}[d]&\\
&*\txt{}\ar@{-}[rr]\ar@{-}[d]&*\txt{}&*\txt{}\ar@{-}[d]\\
a&b&&c\\
}
\\
\hspace*{-22mm}
\underbrace{\hspace*{15mm}}
\end{array}
$$
\caption{Dendrogram in which 
$b,c$ are closer to each other than to $a$. It contains a subset which is not a cluster.} \label{non-clust}
\end{figure}

\begin{figure}
$$
\begin{array}{c}
\xymatrix{
*\txt{}\ar@{-}[rr]\ar@{-}[d]&*\txt{}\ar@{-}[d]&*\txt{}\ar@{-}[d]\\
a&b&c
}
\\
\hspace*{-11mm}\underbrace{\hspace*{16mm}}
\end{array}
$$
\caption{Dendrogram with equidistant data   contains  non-ver\-ti\-cial clusters.} \label{nonvertclust}
\end{figure}
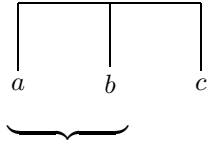
\end{exa}

A {\em clustering} of $X$ is a collection $\mathscr{C}$ of
disjoint clusters of $X$ whose union is the whole dataset $X$.
It is called {\em verticial}, if it consists entirely of verticial clusters.

\medskip
Notice that the definition of  cluster depends on the dataset $X$.
In particular, a non-verticial cluster can be made into a disk by
deleting some data from $X$. E.g.\ in Figure \ref{nonvertclust}
the removal of $c$ from the dataset turns
 $C=\mathset{a,b}$ into a verticial cluster. 
In general, if $\mathscr{C}$ is a clustering of $X$, and $Y\subseteq X$,
then $\mathscr{C}_Y:=\mathset{C\cap Y\mid C\in\mathscr{C}}$
is the {\em restriction of $\mathscr{C}$ to $Y$}.
This motivates us to consider only the case of verticial clusterings.

\begin{Ass}
All clusterings we consider are verticial on some specified (non-empty) subsets
of $X$.
\end{Ass}

\section{The split-LBG algorithm} \label{sec-LBG}
Here, we review briefly the classical split-LBG algorithm. Details can be
found in \cite{LBG1980}.

\bigskip
Let $X=\mathset{a_1,\dots,a_n}$ and $C=\mathset{c_1,\dots,c_k}$  be sets of
vectors in $\mathds{R}^m$,
where $X$ is considered as the {\em data} and $C$ are the prespecified 
{\em cluster centers}. The task is classically to find a partition
$\mathscr{O}=\mathset{\Omega_c\mid c\in C}$ of $X$
into $k$ clusters $\Omega_c$ minimising the energy
$$
E(\mathscr{O},C)=\sum\limits_{c\in C}\sum\limits_{a\in\Omega_c}d(a,c),
$$
where $d(x,y)$ is Euclidean distance in $\mathds{R}^m$.
In fact, the split-LBG method works with varying $C$ by
alternatively constructing partitions and then replacing each $c\in C$ by
two new centers $c+{\bf \varepsilon},c-{\bf \varepsilon}$, where ${\bf
  \varepsilon}$ is a 
perturbation vector in $\mathds{R}^m$ of small norm. 
From these, a new partition is constructed, etc.

\section{Split-LBG in the $p$-adic case} \label{sec-LBGp}
In \cite{B-PXK2001} it was observed that the split-LBG method has no direct
translation using the $p$-adic metric.
Here, we describe a $p$-adic modification of the task from the previous section.

\bigskip
Let $X=\mathset{x_1,\dots,x_n}\subseteq K$ be some data consisting of $n$
$p$-adic numbers, and fix a number $k$. 
The task is to find a clustering $\mathscr{C}=\mathset{C_1,\dots,C_\ell}$
of $X$ with $\ell\le k$, and for each cluster $C\in\mathscr{C}$ a {\em center} $a_C\in C$, minimising the expression
$$
E_p(X,\mathscr{C},{\bf a}):=\sum\limits_{C\in\mathscr{C}} \sum\limits_{x\in C}\absolute{x-a_C}_K,
$$ 
where ${\bf a}=(a_C)_{C\in\mathscr{C}}$ is the sequence of cluster centers.

\medskip
Note that, by the ultrametric
property of $\absolute{\ }_K$,  cluster centers
can (and will) always be chosen within $X$. This has already been taken care
of in the definition of the task. Note further that, unlike in the Archimedean
setting, cluster centers are in general not uniquely defined by their
corresponding clusters.


\medskip
The most significant difference to the Archimedean case is
given by the fact that
in the $p$-adic situation, it does not make sense to choose a cluster center
{\em a priori}, as illustrated in Example \ref{clustcentre}.
Therefore, the  order is reversed: first find a good partition, and then find
corresponding
cluster centers.

\begin{exa} \label{clustcentre}
Let $\mathset{a,b,c}$ be some data with corresponding dendrogram as in Figure
\ref{non-clust}.
Then choosing $a,b$ as centers leads to the clustering
$\mathscr{C}=\mathset{\mathset{a},\mathset{b,c}}$,
whereas the choice $b,c$ leads either to
$\mathscr{C}'=\mathset{\mathset{b,c}}$,
 $\mathscr{C}''=\mathset{\mathset{a,b,c}}$, or to
$\mathscr{C}'''=\mathset{\mathset{b},\mathset{c}}$. But $\mathscr{C}'$
and $\mathscr{C}'''$ are not clusterings
of $\mathset{a,b,c}$, while $\mathscr{C}''$ is.
And both $\mathscr{C}'$ and $\mathscr{C}''$ each consist of one
cluster
containing the two prescribed centers instead of two distinct clusters as
should be the case classically.
\end{exa}

Last but not least, we will not give a global solution to the task in the $p$-adic case,
but
find certain types of local minima of $E_p$ in a sense which will become clear in the
following subsection.

\subsection{Some definitions} \label{somedefs}

An important tool in the classification of $p$-adic data
 $X\subseteq K$ is its
dendrogram $D(X)$.  In contrast to the Archimedean situation, it is uniquely determined by the data
(cf.\ \cite{Brad-JoC,Brad-TCJ}).
We view $D(X)$ as a {\em rooted metric tree}.
This means that it has a  root $v_0$, and all edges are oriented away from 
$v_0$ and are assigned a length which is either positive real or infinite.  
The root $v_0$ corresponds to the top cluster consisting of the whole data
$X$. 
The vertices correspond to clusters containing at least
two points from $X$. 
An edge $e$ of $D(X)$  
connecting two vertices is always bounded. 
The individual points of $X$ correspond uniquely to the {\em ends}
of the tree $D(X)$. We do not view the data $X$ as part of the tree $D(X)$, but
as its boundary. Hence, any $x\in X$ sits at the one extreme of an unbounded edge.
Our viewpoint is probably in contrast to most others 
on hierarchical classification, where data correspond to terminal vertices
of dendrograms.
However, we argue in our favour
that the dendrogram should reflect hierarchic   
approximations of data   by  clusters (vertices in $D(X)$)
or, more generally,  by initial terms in some $p$-adic expansion for data
(points in $D(X)$). 
We refer to
\cite{Brad-JoC,Brad-TCJ}
for a more detailed description of $p$-adic dendrograms.

 
\smallskip
Given some vertex $v$ of $D(X)$, let 
$\children(v)$  denote the set of edges emanating from $v$ (i.e.\ not towards
$v_0$),
and let $\#\children(v)$ be its cardinality. By abuse of notation, we will
identify
$\children(v)$ with the set of vertices and ends attached to the edges in $\children(v)$.

Now, an upper bound for the contribution to $E_p$ of a cluster $C_v$, represented by some
vertex or end $v$ is
$$
\mu(v):=\mu(C_v)=\max\mathset{\absolute{x-y}_K\mid x,y\in C_v}.
$$
As a side remark, note that this is nothing but the Haar measure of $K$ evaluated in the $p$-adic
disk $D_v\subseteq K$ corresponding to $v$. In any case, if $v$ is an end then $\mu(v)=0$,
otherwise $\mu(v)>0$.

Given a set $V$ of vertices or ends  of $D(X)$, we set
\begin{align}
E(V):=\sum\limits_{v\in V}(\#C_v-1)\cdot\mu(v), \label{verclusten}
\end{align}
and also write $E(v_1,\dots,v_b)$ in the case that
$V=\mathset{v_1,\dots,v_b}$.
Applying this to $\children(v)$ for a vertex $v$, we obtain:
\begin{align}
E(\children(v))\le E(v). \label{children-diminish-energy}
\end{align}

The following remark shows that minimising $E(V)$ does make sense for our task:

\begin{rem}
Given a clustering $\mathscr{C}=\mathset{C_v\mid v\in V}$, where $V$ is the
corresponding set of vertices,  for any choice of
$\alpha_v\in C_v$ it holds true that
$$
E_p(X,\mathscr{C},{\bf a})\le E(V)=:E(\mathscr{C}),
$$
where ${\bf a}=(\alpha_v)_{v\in V}$.
\end{rem}

Let $\mathfrak{X}_k(Y)$ be the set of all clusterings $\mathscr{C}$ of $X$
with cardinality $\ell\le k$ whose restriction to $Y$ is verticial. On the set
\begin{align}
\mathfrak{X}=\bigcup\limits_{k\in \mathds{N}}\bigcup\limits_{Y\subseteq
  X}\mathfrak{X}_k(Y), \label{setofclusts}
\end{align}
of all clusterings, we define a partial ordering $\le$ (called {\em refinement}) as follows:
$$
\mathscr{C}'\le \mathscr{C},
$$
if all $C\in\mathscr{C}$ are of the form
$
C=\bigcup\limits_{i\in I} C'_i$
with $C'_i\in\mathscr{C}'$
$(i\in I)$.


\smallskip
Let $C_v$ be the smallest verticial cluster containing a given cluster $C$.
Then we can define the functional
$$
E\colon \mathfrak{X}\to\mathds{R},\;\mathscr{C}\mapsto\sum\limits_{C\in\mathscr{C}}(\#C-1)\cdot\mu(C_v),
$$
and observe that this obviously generalises $E(V)$ from (\ref{verclusten}):

\begin{lem} 
If $\mathscr{C}\in\mathfrak{X}$ is verticial, then
$$
E(V)=E(\mathscr{C}), 
$$
where $V$ is the vertex set associated to $\mathscr{C}$.
\end{lem}

\begin{lem} \label{monotonics}
$E$ is strictly monotonic:
\begin{align*}
\mathscr{C}'\le\mathscr{C}\Rightarrow E(\mathscr{C}')\le E(\mathscr{C}), 
\end{align*}
and if $\mathscr{C}'\le\mathscr{C}$ are not equal, then $E(\mathscr{C}')< E(\mathscr{C})$.
\end{lem}

\begin{proof}
Assume $C=\bigcup\limits_{i\in I}C_i'\in\mathscr{C}$ with $C_i'\in\mathscr{C}'$.
Then
$$
\sum\limits_{i\in I}\#(C_i'-1)\cdot\mu(C'_{i,v})
\le 
\sum\limits_{i\in I}\#(C_i'-1)\cdot\mu(C_v)
\le
(\#C-1)\cdot\mu(C_v),
$$
where the first inequality holds true, because all $C_i'$ are contained in
$C$. The second inequality is strict, if $I$ contains more than one element.
That is the case for some $C$, if $\mathscr{C}\neq\mathscr{C}'$.
\end{proof}

We denote by $E_{k,Y}$
the restriction of $E$ to
$\mathfrak{X}_k(Y)$%
. The following is immediate:

\begin{lem}
Let $\mathscr{C}$ and $\mathscr{C}'$ minimise $E_{k,Y}$ and $E_{k',Y}$, respectively.
Then 
$$
k\le k'\Rightarrow E(\mathscr{C}')\le E(\mathscr{C}).
$$
\end{lem}

%

%
%

\subsection{The verticial clustering algorithm} \label{sec-pclust}
The general strategy which we follow is to refine a given clustering of
$X$ in the ``direction'' which yields the lowest value of $E_p$ after
splitting
a vertex. The term ``direction'' refers to the refinement ordering on
$\mathfrak{X}$,
and we follow the possible ``gradients'' from a given point $\mathscr{C}\in\mathfrak{X}$.
Concretely, this means splitting a vertex with highest energy contribution.
In Section \ref{anyprime}, we will see that the  terms in quotation marks here
can be taken ad literam.

\medskip
In this subsection, we deal with 
verticial clusterings only.
We can now formulate:

\begin{alg}[Verticial clustering]  \rm \label{minclustp}
{\em Input}. $p$-adic data $X\subseteq K$ with $\#X \ge 2$, and upper bound
$k\ge 1$ for number of
clusters. 

\medskip\noindent
{\em Step $0$}. Compute $b=\#\children(v_0)$ and $E(v_0)=\mu(v_0)$.

\medskip\noindent
{\em Step $1$}. If $b>k$, then terminate.
Otherwise, compute $E(\children(v_0))$ which is not greater than $ E(v_0)$ by
(\ref{children-diminish-energy}). Further identify the set of vertices
$V_1:=\children(v_0)\cap \Vertices(D(X))$.

\medskip\noindent
{\em Step $N$}. Assume that from the previous step, we are given some family
 $\mathscr{V}_{N-1}=\mathset{V_{N-1}^{(i)}}$ of sets consisting of
$b_{N-1}^{(i)}\le k$ vertices, respectively. If for all
$i$ and all $v\in V_{N-1}^{(i)}$ it holds true that $b_v^{(i)}:=b_{N-1}^{(i)}+\#\children(v)>k$,
then terminate.

Otherwise, find all $i$ and all $v\in V_{N-1}^{(i)}$ such that
$E(W_v^{(i)})$ 
is smallest possible,
where $W_v^{(i)}:=\children(v)\cup V_{N-1}^{(i)}\setminus\mathset{v}$
satisfies $\#W_v^{(i)}\le k$.
Again, by (\ref{children-diminish-energy}), it holds true that
$$
E(W_v^{(i)})\le E(V_{N-1}^{(i)}).
$$
Extract this new family $\mathscr{V}_N$ of  vertex sets together with the
lower energy value $E_N=E(W)$ for $W\in\mathscr{V}_N$.

\medskip\noindent
{\em Output.} A family of clusterings $\mathset{\mathscr{C}_{i}\mid i\in I}$ (corresponding to the vertex sets in the
last step) for which $E=E(\mathscr{C})$ is locally minimal, together with the value of $E$.  
\end{alg}

\subsection{$p$-adic cluster centers}

The next objective
 is to find cluster centers
with respect to the energy functional. Assume that we are given a fixed cluster 
$C=\mathset{a_1,\dots,a_n}\subseteq K$. We wish to find some $\alpha\in C$
which minimises
$$
\epsilon(\alpha):=E_p(C,\mathscr{C},\alpha)=\sum\limits_{a\in C}\absolute{a-\alpha}_K,
$$
where $\mathscr{C}=\mathset{C}$.

A {\em branch} $B$ of a rooted tree $(T,v)$ is a maximal subtree of $T\setminus\mathset{v}$.
It has a root $v_B$ among the vertices of $\children(v)$.
Let $\mathcal{B}(T)$ denote the set of branches of $(T,v)$. In the case of our
dendrogram
$D(C)$, we will write $\mathcal{B}(C)$, instead of $\mathcal{B}(D(C))$.
The branches induce a natural partition of $C$:
$$
C=\bigcup\limits_{B\in\mathcal{B}(C)}C_B
$$
into a disjoint union of  $C_B=\Ends(B)$.

\begin{lem} \label{lem-energy-branch}
Let $\alpha\in C$, and $B_\alpha\in \mathcal{B}(C)$ the branch containing
$\alpha$ as an end, and $C_\alpha=C_{B_\alpha}$. Then
\begin{align} \label{energy-branch}
\epsilon(\alpha)=\#(C\setminus C_\alpha)\cdot\mu(v_0)+E_p(C_\alpha,\mathscr{C}_\alpha,\alpha),
\end{align}
where $\mathscr{C}_\alpha=\mathset{C_\alpha}$.
\end{lem}

\begin{proof}
Together with the identity:
$$
\sum\limits_{a\in C_\alpha}\absolute{a-\alpha}_K=E_p(C_\alpha,\mathscr{C}_\alpha,\alpha),
$$
this follows easily by looking at the tree $D(C)$. 
\end{proof}

\begin{lem} 
Assume the notations as in Lemma \ref{lem-energy-branch}.
It holds true that
\begin{align}
\frac{\epsilon(\alpha)}{\mu(v_0)}=N_\alpha+O(p^{\nu_\alpha}) \label{O4epsilon}
\end{align}
with $N_\alpha\in\mathds{N}$ and $\nu_\alpha<0$.
\end{lem}

Equation (\ref{O4epsilon}) means that $\frac{\epsilon(\alpha)}{\mu(v_0)}$
is a natural number plus some small term given as a multiple of $p^{\nu_\alpha}$. 

\begin{proof}
Set $N_\alpha=\#(C\setminus C_\alpha)$, and notice that
\begin{align}
E_p(C_\alpha,\mathscr{C}_\alpha,\alpha)\le\#C_\alpha\cdot\mu(v_\alpha), \label{energybound}
\end{align}
where $v_\alpha$ is the root of $B_\alpha$. The claim now follows from the
obvious inequality
$\mu(v_\alpha)<\mu(v_0)$.
\end{proof}

Now, we can formulate our algorithm:

\begin{alg}[Cluster centers] \rm \label{centerclustp}
{\em Step $1$}. Find all branches $B^{(1)}\in \mathcal{B}(C)$ with largest
value of $\#C_{B^{(1)}}$. Extract 
those  clusters $C_{B^{(1)}}$ for which $\mu(v_{B^{(1)}})$ is minimal, and
the number 
$$
c_1=\max\mathset{\#C_{B^{(1)}}\mid B^{(1)}\in \mathcal{B}(C)}.
$$

\medskip\noindent
{\em Step $N$}. Assume that in the previous step, a list of 
clusters
$C_{B^{(N-1)}}$, and a  number $c_{N-1}$ is produced.
Find all branches $B^{(N)}$ of the rooted trees $D(C_{B^{(N-1)}})$
with largest possible value $c_N$
of $\#C_{B^{(N)}}$.
Extract those clusters $C_{B^{(N)}}$ minimising $\mu(v_{B^{(N)}})$, together with $c_N$.

\medskip\noindent
At some point, there will be a {\em Step $N'$}  in which  the trees $D(C_{B^{(N)}})$ have only one
vertex  each. The  procedure terminates thus:

\medskip\noindent
{\em Output}. A list $(C_i)_{i\in I}$ of those clusters from Step $N'$ with
minimal value of $\mu(v_i)$, where $v_i$ is the vertex of $D(C_i)$. 
\end{alg}

\begin{thm}
Let $C'=C_{N'}\subseteq C$ be a cluster produced by performing  Algorithm
\ref{centerclustp}.
Then  any $\alpha\in C'$ is a center of $C$ with respect to $E_p$.
\end{thm}

\begin{proof}
Let $C=C_0\supseteq C_1\supseteq \dots C_{N'}=C'$ be a strictly decreasing chain
of clusters produced by the $N'$ steps of Algorithm \ref{centerclustp}.
Let the corresponding cardinalities be $c_0,\dots,c_{N'}$.
By applying Lemma \ref{lem-energy-branch}, it holds true that
\begin{align}
\epsilon(\alpha)=c_{N'}\cdot\mu(v_{N'})+\sum\limits_{i=1}^{N'}(c_{i-1}-c_i)\cdot\mu(v_{i-1}),
\end{align}
where $v_j$ is the root of the corresponding branch from {\em Step j}. 
The minimality of $\epsilon(\alpha)$ is guaranteed by (\ref{O4epsilon}),
applied to 
each step. Notice, that we have used the obvious fact that for $C'$, the inequality (\ref{energybound}) is an equality.
\end{proof}

\subsection{Quasi-verticial clustering}

The two previous subsections already lead to a $p$-adic  algorithm for
verticial clusterings and their centers. In this case, subdividing a cluster
$C_v$ means to make as many subclusters as there are elements in $\children(v)$.
In the case that e.g.\ there are many singletons, this can be a disadvantage.
Hence removing singletons provides more flexibility in that the bigger
subclusters
can either be merged or kept distinct. Even greater flexibility can be achieved
if almost indistinguishable clusters are treated as singletons.

\begin{dfn}
Fix some real $\varepsilon>0$.
A verticial cluster $C_v\subseteq X$ with corresponding vertex $v$ is called a
{\em quasi-singleton}
for $\varepsilon$, if $E(v)<\varepsilon$. 
\end{dfn}

When we speak of a quasi-singleton, we mean a quasi-singleton for some
$\varepsilon$  known from the context.

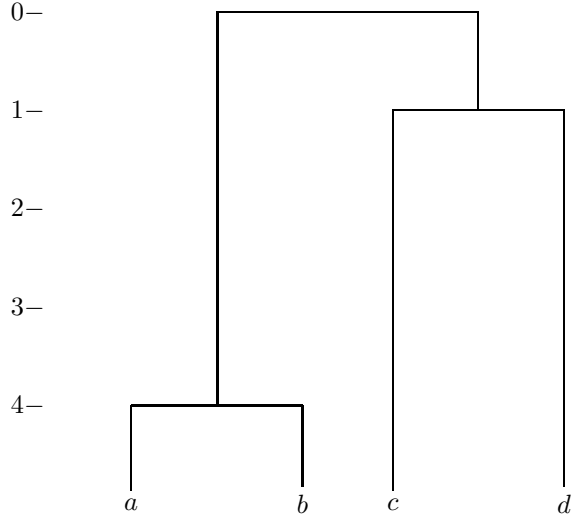
\begin{figure}[h]
$$
\xymatrix{
0-&&*\txt{}\ar@{-}[rrr]\ar@{-}[dddd]&&&*\txt{}\ar@{-}[d]&\\
1-&&&&*\txt{}\ar@{-}[rr]\ar@{-}[dddd]&*\txt{}&*\txt{}\ar@{-}[dddd]\\
2-&&&&&&\\
3-&&&&&&\\
4-&*\txt{}\ar@{-}[rr]\ar@{-}[d]&*\txt{}&*\txt{}\ar@{-}[d]&&&\\
&a&&b&c&&d
}
$$
\caption{Dendrogram with quasi-singleton $\mathset{a,b}$ for $p^{-4}<\varepsilon\le p^{-1}$.} \label{quasising}
\end{figure}

\begin{exa}
The dendrogram in Figure \ref{quasising} contains a quasi-singleton
$\mathset{a,b}$,
if we set $\mu(v)=p^{-\ell}$ for vertex $v$ at level $\ell$ (indicated by the
number at the left), and $p^{-1}<\varepsilon\le p^{-1}$. 
For this choice of $\varepsilon$, the cluster $\mathset{c,d}$ is not
a quasi-singleton. But this is the case for larger $\varepsilon$.
\end{exa}

Clearly, any singleton is a quasi-singleton for any $\varepsilon$.
Since we are working with a fixed $p$-adic field $K$, it is possible to choose
$\varepsilon$ 
so small that the quasi-singletons are precisely the singletons of
our given dataset $X$.

\medskip
The algorithm we propose in the following removes  quasi-singletons
in order to continue with verticial clusterings.
For this, we fix some notation: When referring to a subset $Y$ of our dataset
$X$,
we will indicate this by the subscript $Y$. E.g.\
$\children_Y(v)$ means the set of edges
in $D(Y)$ going out from $v$. Simliarly, with $\mu_Y(V)$, $E_Y(V)$ etc.

\begin{alg}[Quasi-verticial clustering]\rm \label{quasivertclust} 
{\em Input.} Data $X_0:=X\subseteq K$, 
and numbers $k_0:=k\ge 1$, $\varepsilon>0$.

\smallskip\noindent
{\em Step 1.} Remove from $D(X)$ all $v\in\children_{X_0}(v_0)$ corresponding to
 quasi-singletons for $\varepsilon$. Let $s_1$ be the number of vertices removed.
 Extract corresponding reduced dataset $X_1\subseteq X_0$,
as well as
$\children_{X_1}(v_0)$,  $E_{X_1}(v_0)=\mu_{X_1}(v_0)$, and $k_1:=k-s_1$. 

\medskip\noindent
{\em Step $N$.} Assume that in the previous step, we are given 
 a quadruple of  families 
$$
(\mathscr{V}_{N-1},\mathscr{X}_{N-1}, E_{N-1},\mathscr{K}_{N-1})
$$ 
of 
 sets $V\in\mathscr{V}_{N-1}$ of vertices in $D(X)$, 
datasets
 $X(V)\in \mathscr{X}_{N-1}$,
 an energy value
 $E_{N-1}=E_{X(V)}(V)$, and numbers
 $k_{N-1}(V)\le k$
  (where $V\in\mathscr{V}_{N-1}$).
Remove for all $V\in\mathscr{V}_{N-1}$ from $D(X(V))$ all vertices in
$\children_{X(V)}(v)$ corresponding to
$s_N(v)$ quasi-singletons,
where $v\in V$. Find all $V\in\mathscr{V}_{N-1}$ and $v\in V$ such that
\begin{enumerate}
\item $k_{N-1}(V)-s_N(v)\ge 0$, and
\item $E_{X(V)}(W_v)<E_{N-1}$ is smallest possible, 
\end{enumerate}
where $W_v:=\children(v)\cup V\setminus\mathset{v}$. 
Extract corresponding quadruple of families 
$$
(\mathscr{V}_N,\mathscr{X}_N,E_N,\mathscr{K}_N)
$$
 of 
new vertex sets $W_v$,
reduced datasets $X(W_v)\subseteq X(V)$,
 energy value $E_N=E(W_v)$, and $k_N(W_v):=k_{N-1}(V)-s_N(v)$.

\medskip\noindent
{\em Output.}
A list of clusterings consisting of quasi-singletons for $\varepsilon$ and  clusters
produced above by collecting the remnants in each step.
\end{alg}

\begin{rem}
The output clusterings of Algorithm \ref{quasivertclust} all have energy of
the form
$$
E+O(p^\alpha),
$$
where $E$ is independent of the clustering, and $\alpha<0$ is small.
\end{rem}

We can now put things together in order to find clusterings in different ways:

\begin{alg}[(Quasi-)Verticial split-LBG$_p$]\rm \label{splitLBGp}
{\em Input.} As in Algorithm \ref{minclustp} (resp.\ Algorithm \ref{quasivertclust}).

\medskip\noindent
{\em Step 1.} Perform Algorithm \ref{minclustp} (resp.\ Algorithm \ref{quasivertclust}).

\medskip\noindent
{\em Step 2.} Perform Algorithm \ref{centerclustp} for each cluster occurring
in each clustering given out in the previous step.

\medskip\noindent
{\em Output.} A list $(\mathscr{C}_i,({\bf a}_j^{(i)})_{j\in J})_{i\in I}$ of
$E$-suboptimal clusterings  with corresponding list of 
$E$-center vectors $({\bf a}_j^{(i)})_{j\in J}$ for clustering $\mathscr{C}_i$.  
\end{alg}

Both subroutines, Algorithms  \ref{minclustp} and \ref{centerclustp}, boil
down to counts and evaluations of $\mu(v)$ for vertices $v$. 
Therefore, we remark:


\section{Dependence on the choice of the prime $p$} \label{anyprime}

A natural issue is, how the outputs of the algorithms introduced in the
previous sections depend
on the choice of the prime number $p$. We will prove a finiteness result.

\bigskip
Recall that the energy of a verticial cluster $C_V$ is of the form
\begin{align}
E(C_V)= A\cdot p^{-\nu} \label{energyterm}
\end{align}
with  natural numbers $A$ and $\nu$, and is additive on disjoint unions of clusters.
Splitting a cluster is performed by replacing vertex $v$ by the vertex set
$\children(v)$, and the change in energy is given by
$$
E_{\rm new}=E_{\rm old}-E(C_v)+E(\children(v)),
$$
i.e.\ the difference is 
$$
\delta_vE_p:=E(C_v)-E(\children(v)).
$$
Our approch towards minimising $E_p$ is to refine the given clustering
in the direction of largest $\delta_vE_p$. Now, the quantity $\delta_vE_p$
depends on the prime number $p$ as shown by (\ref{energyterm}).
This means that different $p$ can result in different
rankings  of the   vertices by the order in 
which they are split. 
We call this the {\em $p$-ranking} of the vertices of $D(X)$.

\begin{exa}
Assume we want to find verticial clusterings of  data 
$$
X=\mathset{x_1,\dots,x_{13}}
$$ 
having
underlying dendrogram as in Figure \ref{dendro}.
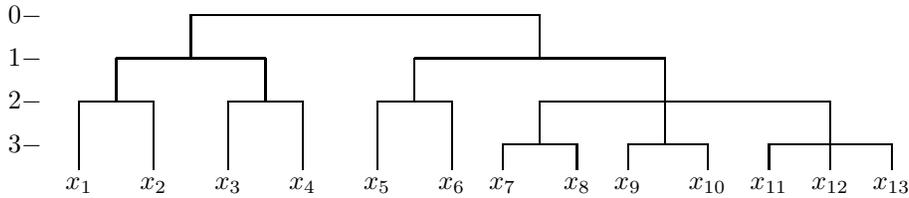
\begin{figure}[h]
$$
\xymatrix@=3pt{
0-&&&&*\txt{}\ar@{-}[rrrrrrrrr]\ar@{-}[d]&&&&&&&&&*\txt{}\ar@{-}[d]&&&&&&
\\
1-&&*\txt{}\ar@{-}[rrrr]\ar@{-}[d]&&*\txt{}&&*\txt{}\ar@{-}[d]&&&&*\txt{}\ar@{-}[rrrrrr]\ar@{-}[d]&&&*\txt{}&&&*\txt{}\ar@{-}[d]&
\\
2-&*\txt{}\ar@{-}[rr]\ar@{-}[dd]&*\txt{}&*\txt{}\ar@{-}[dd]&&*\txt{}\ar@{-}[rr]\ar@{-}[dd]&*\txt{}&*\txt{}\ar@{-}[dd]&&*\txt{}\ar@{-}[rr]\ar@{-}[dd]&*\txt{}&*\txt{}\ar@{-}[dd]&&
*\txt{}\ar@{-}[rrrrrr]\ar@{-}[d]&&&*\txt{}\ar@{-}[d]&&&*\txt{}\ar@{-}[d]&&
\\
3-&&&&&&&&&&&&*\txt{}\ar@{-}[rr]\ar@{-}[d]&*\txt{}&*\txt{}\ar@{-}[d]&*\txt{}\ar@{-}[rr]\ar@{-}[d]&*\txt{}&*\txt{}\ar@{-}[d]&*\txt{}\ar@{-}[rr]\ar@{-}[d]&*\txt{}\ar@{-}[d]&*\txt{}\ar@{-}[d]
\\
&x_1&&x_2&&x_3&&x_4&&x_5&&x_6&x_7&&x_8&x_9&&x_{10}&x_{11}&x_{12}&x_{13}
}
$$
\caption{A dendrogram.} \label{dendro}
\end{figure}
Consider the  vertices $a,b,c,d$ in the underlying rooted vertex tree as depicted in
Figure \ref{tree4dendro}.
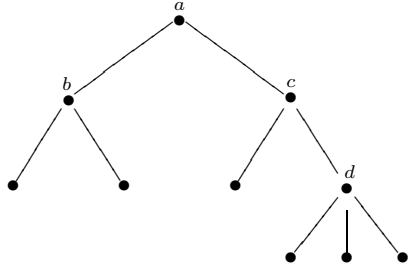
\begin{figure}[h]
$$
\xymatrix@=15pt{
&&&*\txt{\raisebox{3pt}{$\stackrel{a}{\bullet}$}}\ar@{-}[dll]\ar@{-}[drr]&&&&\\
&*\txt{\raisebox{6pt}{$\stackrel{\!b}{\bullet}$}}\ar@{-}[dl]\ar@{-}[dr]&&&&*\txt{\raisebox{6pt}{$\stackrel{c}{\bullet}$}}\ar@{-}[dl]\ar@{-}[dr]&&\\
*\txt{$\bullet$}&&*\txt{$\bullet$}&&*\txt{$\bullet$}&&*\txt{\raisebox{6pt}{$\stackrel{\;d}{\bullet}$}}\ar@{-}[dl]\ar@{-}[d]\ar@{-}[dr]&\\
&&&&&*\txt{$\bullet$}&*\txt{$\bullet$}&*\txt{$\bullet$}
}
$$
\caption{Vertex tree underlying Figure \ref{dendro}.}\label{tree4dendro}
\end{figure} 
Then Table \ref{vertexrank} shows the different $p$-rankings of these vertices for
$p=2$, $3$ and $5$.
\begin{table}[h]
\rm
\begin{tabular}{c|c|c}
Rank&Vertex&$\delta_vE_2$\\\hline
1.&a&\rule[-5pt]{0pt}{15pt}$\frac{11}{2}$\\\hline
2.&c&\rule[-5pt]{0pt}{15pt}$\frac{9}{4}$\\\hline
3.&b&$2$\\
&d&$2$\\\hline
\multicolumn{3}{c}{\rule[-5pt]{0pt}{20pt}$p=2$}
\end{tabular}
\hfill
\begin{tabular}{c|c|c}
Rank&Vertex&$\delta_vE_3$\\\hline
1.&a&\rule[-5pt]{0pt}{15pt}$\frac{22}{3}$\\\hline
2.&c&\rule[-5pt]{0pt}{15pt}$\frac{17}{9}$\\\hline
3.&b&\rule[-5pt]{0pt}{15pt}$\frac{7}{9}$\\\hline
4.&d&\rule[-5pt]{0pt}{15pt}$\frac{16}{27}$\\\hline
\multicolumn{3}{c}{\rule[-5pt]{0pt}{20pt}$p=3$}
\end{tabular}
\hfill
\begin{tabular}{c|c|c}
Rank&Vertex&$\delta_vE_5$\\\hline
1.&a&\rule[-5pt]{0pt}{15pt}$\frac{44}{5}$\\\hline
2.&c&\rule[-5pt]{0pt}{15pt}$\frac{44}{25}$\\\hline
3.&b&\rule[-5pt]{0pt}{15pt}$\frac{13}{25}$\\\hline
4.&d&\rule[-5pt]{0pt}{15pt}$\frac{26}{225}$\\\hline
\multicolumn{3}{c}{\rule[-5pt]{0pt}{20pt}$p=5$}
\end{tabular}
\caption{Vertex rankings for Figure \ref{dendro}.}\label{vertexrank}
\end{table}
\end{exa}

\begin{thm}
For all but finitely many primes, the $p$-rankings of the vertices of a given
dendrogram $D(X)$ belonging to  data $X$ taken from a fixed $p$-adic field
are the same.
\end{thm}

\begin{proof}
The energy gradient for a vertex $v$ can be written as
$$
\delta_vE_p=P_v(t)|_{t=\frac{1}{p}}
$$
for some  polynomial $P_v(t)$ whose coefficients are natural numbers.
By dividing off powers of $t$, we may assume that $P_v(t)$ has a non-zero
constant term, hence that $P_v(0)>0$. By the considerations from the previous
sections,
we know that 
\begin{align}
0<P_v\left(\frac{1}{p}\right)<P_v(0) \label{decrease}
\end{align}
for all primes $p$. By viewing $P_v(t)$ as a continuous function on the
intervall $[0,1/2]$, we see from the right inequality in (\ref{decrease}) that
$P_v(t)$ must be decreasing on some interval $[0,x]$ with positive
$x\le\frac{1}{2}$ sufficiently small.
It follows that the sequence of values
$P_v\left(\frac{1}{p}\right)$ for prime $p\to\infty$ converges to
$P_v(0)$. Since that limit equals $E(v)$
on the maximal subtree of $D(X)$ having $v$ as its root, we have proven
$$
\lim\limits_{p\to\infty}P_v\left(\frac{1}{p}\right)=E(v).
$$
In other words, for sufficiently large prime $p$, the vertex gradient can be
approximated
by the vertex energy. Hence the ranking of the vertices is approximatively the
ranking
of the numbers
\begin{align}
\frac{E(v)}{p^{\ell(v)}}, \label{energyrank} 
\end{align}
where $\ell(v)$ depends on the level of $v$ in the dendrogram.
The latter ranking does not change once $p$ is sufficently large.
Hence, for large $p$ the vertex ranking does not change.
\end{proof}

\begin{rem}
Notice from (\ref{energyrank}) that using a large prime number tends to force
splitting vertices 
 higher up in the hierarchy underlying the dendrogram. On the other
hand, taking a small prime number allows to split also clusters containg lots of
data
at low levels in the hierarchy.
\end{rem}

\begin{thm}
Let $C\subseteq K$ be a cluster. If $a$ is a center of $C$ with respect to
$E_p$
for some prime $p$, then it is a center for all primes. 
\end{thm}

\begin{proof}
From Lemma (\ref{lem-energy-branch}) it follows that
$$
\epsilon_p(a)=E_p(C,\mathscr{C},a)=\sum\limits_{v\in V}\alpha_v\mu(v),
$$
where $V$ is the set of vertices on the path $\gamma$ from the top $v_0$ down to $a$. 
As $\mu(v)=p^{-\ell(v)}$, and the $\ell(v)$ form a strictly increasing
sequence $\ell_0,\dots,\ell_M$  of natural numbers as $v$ proceeds along $\gamma$, it follows that
$\epsilon_p(a)$
is given by evaluating the polynomial
$$
F_\gamma(t)=\sum\limits_{i=0}^M\alpha^\gamma_it^{\ell_i}
$$
in $t=\frac{1}{p}$,
where $\alpha^\gamma_i>0$ equals that number $\alpha_v$ with $v$ such that $\ell(v)=\ell_i$.
Now, $\epsilon_p(a)$ being a minimum means that in the collection 
$$
\mathset{F_\gamma(t)\mid\text{$\gamma$ path $v_0\leadsto X$}}
$$ 
the term $a_0^\gamma t^{\ell_0}$ is of
lowest
degree and that coefficient $\alpha^\gamma_0$ is smallest among those terms of
lowest degree. And this does not depend on the choice of prime $p$.
\end{proof}

\section{$p$-adic learning} \label{sec-learn}

In this section we discuss a learning situation in which some $p$-adic data
$X\subseteq K$ together with a clustering $\mathscr{C}_X$
is used as
a ``training set''. The idea is to classify
new data $Y$ taken from some $p$-adic field $L\supseteq K$ on the basis of $X$
and $\mathscr{C}$. 
Without loss of generality we assume that the two $p$-adic
fields $L$ and $K$ coincide.

\subsection{$p$-adic classifiers}
 Learning can be performed by using a classifier which integrates new data
$y\in Y$ into an existing dendrogram $D(X)$ in order to find a suitable
cluster for $y$. We will define such in the $p$-adic situation.

\smallskip
As it may happen that adjoining a point $y\in Y$ to $X$
increases the size of the smallest $p$-adic disk containing the training data $X$,
we use the point at infinity already introduced in \cite{Brad-JoC}. This
allows to classify those data in $Y$ which cannot be classified on the basis
of $(X,\mathscr{C}_X)$
as belonging to the ``cluster at infinity''.
Our method will  use the extended dendrogram
$$
D_\infty(X)=D(\mathds{P}(X)),
$$
where $\mathds{P}(X)=X\cup\mathset{\infty}$\footnote{Note that $D_\infty(X)$
  is
what is denoted by $D(X)$ in \cite{Brad-JoC,Brad-TCJ}.}. The datum $\infty$
will be depicted  at the end of a path going upwards from $v_0$, whereas all other
data will be  at the end of  paths leading downwards.

\begin{exa}
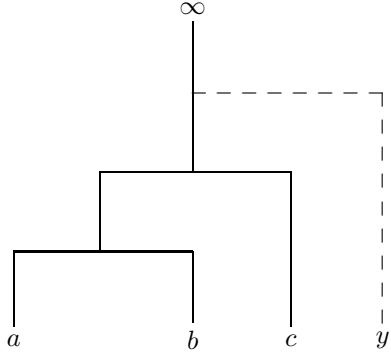
\begin{figure}[h]
$$
\xymatrix{
&&\infty&&\\
&&*\txt{}\ar@{--}[rr]&&*\txt{}\ar@{--}[ddd]\\
&*\txt{}\ar@{-}[rr]\ar@{-}[d]&*\txt{}\ar@{-}[uu]&*\txt{}\ar@{-}[dd]&\\
*\txt{}\ar@{-}[rr]\ar@{-}[d]&*\txt{}&*\txt{}\ar@{-}[d]&&&\\
a&&b&c&y
}
$$
\caption{Dendrogram with cluster at infinity.} \label{learninfty}
\end{figure}
In Figure \ref{learninfty}, some datum $y$ is adjoined to a training dataset
$X=\mathset{a,b,c}$.
As it happens that the distance of $y$ to $X$ is larger than the diameter of
$X$,
the path $v_0\leadsto y$ in the dendrogram $D_\infty(X\cup\mathset{y})$ has a
portion going upwards in direction $\infty$.   
\end{exa}

We call the pair $\mathcal{X}:=(X,\mathscr{C}_X)$ a {\em classification} and have a {\em
  classification map}
$$
\kappa_{\mathcal{X}}\colon \mathds{P}(X)\to\mathscr{C}_X^\infty,\;x\mapsto C_x,
$$
which assigns to each $x\in \mathds{P}(X)$ the cluster $C_x$ containing $x$,
with $$\mathscr{C}_X^\infty=\mathscr{C}_X\cup\mathset{\mathset{\infty}}.$$

Now, let $Z=X\cup Y$. We have the  inclusion map
$\iota\colon\mathds{P}(X)\to \mathds{P}(Z)$ which takes $x\in X$ to itself and $\infty$
to
$\infty$.

\begin{dfn}
A {\em $p$-adic classifier for $Y$ modeled on
  $(X,\mathscr{C}_X)$} is a  map
\begin{align*}
\lambda\colon \mathds{P}(Z)\to \mathscr{C}',
\end{align*}
where $\mathscr{C}'$ is a clustering of $\mathds{P}(Z)$, such that there
exists an injective map $\phi\colon\mathscr{C}_X^\infty\to\mathscr{C}'$ making 
the diagram
$$
\xymatrix{
\mathds{P}(X)\ar[r]^\iota\ar[d]_{\kappa_{\mathcal{X}}}&\mathds{P}(Z)\ar[d]^\lambda\\
\mathscr{C}_X^\infty\ar[r]_\phi&\mathscr{C}'
}
$$ 
commutative. The cluster $C_\infty:=\lambda^{-1}(\phi(\mathset{\infty}))$ is
called
the {\em residue} of $\lambda$. A classifier is called saturated, if $\phi$ is bijective.
\end{dfn}

\begin{rem}
Notice that $\phi$ is unique if it exists. 
\end{rem}

Our first learning algorithm constructs the classifier sequentially
by computing the distance to  cluster centers for $\mathscr{C}_X$.
Let $A=\mathset{a_C\mid C\in\mathscr{C}_X}$ be the set of given cluster
centers
$a_C\in C$. 
Then we have for $y\in Y$ the map
$$
d_y\colon\mathscr{C}_X\to\mathds{R},\; C\mapsto \absolute{y-a_C}_K,
$$
and let $m_y:=\min d_y(\mathscr{C}_X)$. 

The  vertex  $v_y\in D_\infty(A\cup\mathset{y})$ nearest to $y$ can be 
found e.g.\  
using the $p$-adic expansions 
as given by  (\ref{pi-adic-expansion}).
Namely, a vertex corresponds to a disk containing two or more $p$-adic
numbers in $A\cup\mathset{y}$ having common initial terms determined by the radius of the disk.
In geometric terms, traversing along the
 geodesic path $\gamma_y\colon \infty \leadsto y$ until all $a\in A$
have
branched off $\gamma_y$ yields  the vertex $v_y$, and $\mu(v_y)$
is determined 
by the subset $C_{v_y}\subset A$ of those elements branching off precisely in
$v_y$. The length of the path $v_0\leadsto v_y$ gives $m_y$.
And  the map $d_y$ is computed:

\begin{lem}
It holds true that 
$$
\mathscr{C}_y:=d_y^{-1}(m_y)=\mathset{C_a\in\mathscr{C}_X\mid a\in C_{v_y}}.
$$
\end{lem}

\begin{proof}
By what has been said above, the minimum is attained precisely for those
clusters $C\in\mathscr{C}_X$ contained in $C_{v_y}$. Hence $C=C_a$ for some $a\in C_{v_y}$. 
\end{proof}

The task is now to decide into which cluster from $\mathscr{C}_y$ to put $y$.

\begin{alg}\rm \label{plearn}
{\em Input.} A classification $\mathcal{X}_0:=\mathcal{X}=(X,\mathscr{C}_X)$,
a set $A=\mathset{a_C \mid C\in\mathscr{C}_X}$ of cluster elements $a_C\in C$, and a set
$Y\subseteq K$ of cardinality $N$. 

\medskip\noindent
{\em Step 0.} Set $C_\infty:=\mathset{\infty}$.

\medskip\noindent
{\em Step 1.} Take $y:=y_1\in Y$, and compute $v_y$, $m_y$, $C_{v_y}$, $\mathscr{C}_y$,
and $\mu_{v_y}$. 

\smallskip\noindent
{\em Case 1.} If $C_{v_y}=\mathset{a}$, then set
$C_y:=C_a\cup\mathset{y}$ and $A_1:=A$. 

\smallskip\noindent
{\em Case 2.} If $\#C_{v_y}>1$, then find the subset $C^y\subseteq C_{v_y}$ of all elements  whose
nearest vertex in $D_\infty(C_{v_y}\cup\mathset{y})$ equals $v_y$.
If $C^y=\emptyset$, then set $C_y=\mathset{y}$ and $A_1:=A\cup\mathset{y}$.
Otherwise, find all elements $a\in C^y$ with minimal energy $E(C_{a}\cup\mathset{y})$.
If there is more than one such $a$, then $C_y:=\mathset{y}$ and
$A_1:=A\cup\mathset{y}$. 
Otherwise, $C_y:=C_a\cup\mathset{y}$, and $A_1:=A$. 

\smallskip
In any case, produce $Y_1:=Y\setminus\mathset{y}$, $A_1$ and 
 classification $\mathcal{X}_1:=(X_1,\mathscr{C}_{X_1})$,
where $X_1=X\cup\mathset{y}$ and
$\mathscr{C}_{X_1}:=\mathset{C_y}\cup\mathscr{C}_X\setminus\mathset{C_a}$.
Terminate, if $Y_1=\emptyset$.

\medskip\noindent
{\em Step $N$.} Assume that in the previous step, sets $Y_{N-1}$, $A_{N-1}$ and a
classification $\mathcal{X}_{N-1}$ have been
given out. Then perform Step $1$ with $\mathcal{X}:=\mathcal{X}_{N-1}$,  $A:=A_{N-1}$,
and $Y:=Y_{N-1}$. 

\medskip\noindent
{\em Output.} On termination in {\em Step $M$}, an optimal classifier 
$$
\lambda\colon \mathds{P}(X_M)\to\mathscr{C}_{X_M},\;x\mapsto C_x,
$$
modeled on $\mathcal{X}_0$.
\end{alg}

\begin{proof}[Proof of optimality]
In each step $N$, $y_N\in Y_N$ is assigned to the  cluster 
$C\in\mathscr{C}_{X_N}$
with minimal energy $E(C\cup\mathset{y_N})$. 
\end{proof}

\begin{thm}\label{plearnthm}
The outcome of Algorithm \ref{plearn} does not depend on the choice of the set
$A$
of cluster representatives.
\end{thm}

\begin{proof}
The outcome of {\em Step $1$} does not depend on $A$.
\end{proof}

\begin{rem}
A consequence of Theorem \ref{plearnthm} is that Algorithm \ref{plearn} does
indeed effect learning in the sense, that to any $y\in Y$ is assigned a
cluster depending on the already existing clusters. Representing a cluster
by a single element makes learning efficient.
\end{rem}

\subsection{Adaptive learning}
During the learning process\footnote{Or if for some reason one wants to
  perform a variation of split-LBG$_p$ in which centers are computed after each
  clustering step, instead of after termination of clustering.}, it can become
useful to subdivide
big clusters of the extended dataset $X\cup Y$. This is not a problem, as
the old cluster centers can be reused in the new clustering.

\begin{lem} \label{subclustcenter}
Let $C$ be a cluster, and $a\in C$ a center of $C$. Assume that $C'$ is a subcluster
of
$C$ containing $a$,  then $a$ is a center of $C'$.
\end{lem}

\begin{proof}
Clearly, it holds true that
\begin{align}
E_p(C,\mathscr{C},a)\le E_p(C,\mathscr{C},a'), \label{trivineq}
\end{align}
where $\mathscr{C}=\mathset{C}$ and $\mathscr{C}'=\mathset{C}$. 
Assume that  $a'\in C'$ is a center of $C'$.
Now, inequality (\ref{trivineq}) implies that
\begin{align*}
E_p(C',\mathscr{C}',a)
+\sum\limits_{x\in C\setminus  C'}\absolute{x-a}_K
&=E_p(C,\mathscr{C},a)
\\
&\le E_p(C,\mathscr{C},a')
\\
&=E_p(C',\mathscr{C}',a')+\sum\limits_{x\in C\setminus
C'}\absolute{x-a'}_K
\end{align*}
Since, by the cluster property of $C'$, it holds true that 
$$
\absolute{x-a}_K=\absolute{x-a'}_K
$$
for all $x\in C\setminus C'$, it follows that
\begin{align}
E_p(C',\mathscr{C}',a)\le E_p(C',\mathscr{C}',a'), \label{lowerthancenter}
\end{align}
and, because $a'$ is a center of $C'$, this yields an equality in (\ref{lowerthancenter})
i.e.\ $a$ is a center of $C'$.
\end{proof}

\begin{rem}
Notice that Lemma \ref{subclustcenter} does not hold true, if we allow $C'$ to
be an arbitrary subset of $C$. E.g.\ assume in Figure \ref{4dendro} that
$C=\mathset{a,b,c,d}$. Then $a$ is a center of $C$, as can be verified from
the left dendrogram. However, $a$ is not a
center of $C'=\mathset{a,c,d}$, as the right dendrogram reveals.
Namely, in the first case, we compute with $\mathscr{C}=\mathset{C}$
and $\mathscr{C}'=\mathset{C'}$:
\begin{align*}
E(C,\mathscr{C},a)&=E(C,\mathscr{C},b)=\absolute{a-b}_K+2\cdot\absolute{a-c}_K\\
&<\absolute{c-d}_K+2\cdot\absolute{a-c}_K=E(C,\mathscr{C},c)=E(C,\mathscr{C},d),
\end{align*}
and in the second case:
\begin{align*}
E(C',\mathscr{C}',c)&=\absolute{a-c}_K+\absolute{d-c}_K\\
&<2\cdot\absolute{a-c}_K=E(C',\mathscr{C}',a).
\end{align*}

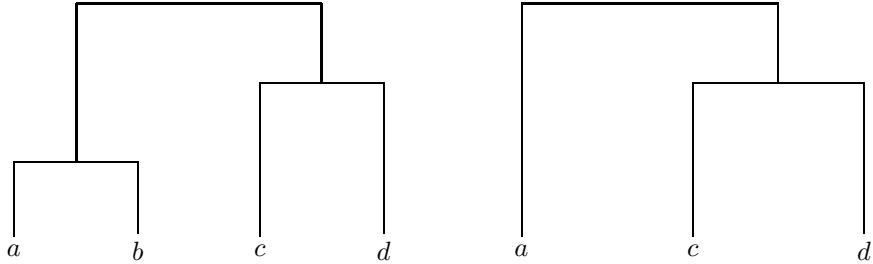
\begin{figure}[h]
$$
\xymatrix@C15pt{
&*\txt{}\ar@{-}[rrrr]\ar@{-}[dd]&&&&*\txt{}\ar@{-}[d]&\\
&&&&*\txt{}\ar@{-}[dd]&*\txt{}&*\txt{}\ar@{-}[dd]\ar@{-}[ll]\\
*\txt{}\ar@{-}[d]\ar@{-}[rr]&*\txt{}&*\txt{}\ar@{-}[d]&&&&\\
a&&b&&c&&d
}
\quad\quad\quad\quad
\xymatrix{
*\txt{}\ar@{-}[rrr]\ar@{-}[ddd]&&&*\txt{}\ar@{-}[d]&\\
&&*\txt{}\ar@{-}[dd]&*\txt{}&*\txt{}\ar@{-}[ll]\ar@{-}[dd]\\
&&&&\\
a&&c&&d
}
$$
\caption{Dendrogram and subdendrogram.}\label{4dendro}
\end{figure}
\end{rem}

At last, we propose the splitting of high-energy clusters accumulated during
the learning process:

\begin{alg}\rm
{\em Input.} $r\ge 0$. Otherwise, as in Algorithm \ref{plearn}.

\smallskip\noindent
Perform Algorithm \ref{plearn} with modification:

\medskip\noindent
{\em Step $N'$.}
Perform {\em Step $N$}. If for $y:=y_N$ it holds true that $E(C_{y})>r$,
then split cluster $C_{y}$ into its maximal proper
subclusters, 
and adjoin to $A_N$ new cluster centers using Algorithm \ref{centerclustp}.
\end{alg}

\section{Conclusion}

A straightforward translation of the split-LBG algorithm to the
situation of classifying $p$-adic
data does not exist. However, if clusterings,  cluster centers and their numbers are
allowed to vary, then the minimisation problem for the $p$-adic energy
functional defined by distances to centers does make sense. 
Sub-optimal
algorithmic solutions to the minimisation problem are presented, in which the choice lies
in whether or not to remove in each step
 quasi-singletons, i.e.\ clusters which are almost singletons because of  their
energy values being lower than a given threshold. The method is
 to find rankings of vertices in the dendrogram associated to the $p$-adic data.
The outcome
 depends on the prime number $p$, but it is shown that for all but  finitely
 many primes the rankings are identical.  
The consequence for applications to data anlaysis is that for fixed prime
$p$,
the classification results do not depend on the $p$-adic representation
of the data, as long as the dendrograms are isomorphic.
Furthermore, the minimising property for
given cluster centers holds true  independently of the prime. This means
that if some datum is a cluster center for one prime, it is a cluster center
for all primes (for which the corresponding cluster is not larger).
Using $p$-adic cluster centers,
one can construct classifiers from given clusterings. This can be applied to
learning situations.

\section*{Acknowledgements}

The author acknowledges support from 
 the DFG-project BR 3513/1-1. Andrei Khrennikov is warmly thanked for
 directing the author's attention to
\cite{B-PXK2001} where the issue of a
$p$-adic form of split-LBG
is raised. Thanks to Fionn Murtagh for remarks leading to improving the
exposition of the article. The anonymous referee is thanked for pointing
out various inaccuracies and errors.
The Institut f\"ur Photogrammetrie und Fernerkundung at University
Karlsruhe
is thanked for the opportunity to write down this article.

\end{document}